\documentclass[letterpaper, 10 pt, conference]{ieeeconf}  %

\IEEEoverridecommandlockouts                              %

\usepackage{amsmath} %
\usepackage{amssymb}  %
\usepackage{multirow}
\usepackage{multicol}
\usepackage{booktabs}
\usepackage{graphicx}
\usepackage{array}
\newcommand{\method}{POIL}
\newtheorem{theorem}{Theorem}
\newtheorem{lemma}[theorem]{Lemma}
\usepackage{pifont}
\usepackage{makecell}
\usepackage{xcolor}
\usepackage{hyperref}
\hypersetup{
  pdftitle={POIL: Point-based One-Shot Imitation Learning with Stable Dynamical Systems},
  pdfauthor={Sang Min Kim, Jinwoo Seo, Hyeongjun Heo, Junho Lee, Yonghyeon Lee, Young Min Kim},
  pdfkeywords={one-shot imitation learning, stable dynamical systems, point-based representation, robot manipulation}
}
\newcommand{\cmark}{\ding{51}}
\newcommand{\xmark}{\ding{55}}

\DeclareMathOperator*{\argmin}{arg\,min}

\title{\LARGE \bf
\method{}: Point-based One-Shot Imitation Learning \\with Stable Dynamical Systems
}

\author{%
Sang Min Kim$^{1}$,
Jinwoo Seo$^{1}$,
Hyeongjun Heo$^{1}$,
Junho Lee$^{1}$,
Yonghyeon Lee$^{2*}$, and
Young Min Kim$^{1*}$%
\thanks{%
$^{1}$Department of Electrical and Computer Engineering,
Seoul National University, Seoul, Republic of Korea.
$^{2}$Department of Artificial Intelligence,
Yonsei University, Seoul, Republic of Korea.
Junho Lee is now with the Autonomous Systems Lab,
Institute of Computer Technology, TU Wien, Vienna, Austria.
$^{*}$Co-corresponding authors:
\texttt{yonghyeon.lee@yonsei.ac.kr} and
\texttt{youngmin.kim@snu.ac.kr}.}%
}

\begin{document}

\maketitle
\thispagestyle{empty}
\pagestyle{empty}

\begin{abstract}
We present \method{}, a point-based one-shot imitation learning framework with stable dynamical systems.
While one-shot imitation avoids collecting extensive demonstrations, successful one-shot manipulation requires not only transferring a demonstrated trajectory to a novel object but also executing it robustly under changing scene conditions, grasp configurations, and external disturbances.
\method{} addresses both problems through a shared representation: a set of 3D points on the object's functional part, used jointly for trajectory transfer and closed-loop execution.
The one-shot transfer from the demonstrated trajectory is enabled with point correspondences.
\method{} grounds the shared functional part with a multi-modal large language model, and transfers the trajectory across viewpoint, pose, and object category changes.
During execution, multi-view tracking observes the same points online, and Point-set BCSDM drives them in closed loop by projecting per-point velocities onto a single rigid-body twist computed from the tracked points alone.
This extends stable dynamical models from an SE(3) pose to a point set without requiring a known 3D model or pose estimator.
We show that at the goal the controller becomes a gradient flow on the classical SO(3) potential, so its terminal phase inherits the almost-global convergence of that potential under a rigid-object assumption.
Across simulation and real-robot experiments, \method{} transfers a single demonstration across object category, grasp pose, and goal geometry, while recovering from external disturbances during execution.
Project page: \url{https://sangminkim-99.github.io/poil/}
\end{abstract}

\section{INTRODUCTION}

While recent data-driven methods perform tasks successfully in a novel environment, one still needs to acquire sufficient demonstration trajectories to adapt general models to the current setting~\cite{zitkovich2023rt, kim2024openvla, o2024open}.
One-shot imitation learning minimizes the requirement to a single demonstration, and aims to generalize the manipulation skills to unseen scenarios~\cite{finn2017one, duan2017one, vosylius2024instant}.
Table~\ref{tab:comparison} compares existing approaches along four axes: \textit{action-labeled pretraining}, \textit{object generalization}, \textit{closed-loop} execution, and a \textit{convergence guarantee}.
The first axis is whether the method requires action-labeled pretraining beyond the single demonstration itself.
Such pretraining is a powerful route to generalization, but the data remains costly to collect, and a method free of it stays applicable where the data is unavailable.
The second axis is how far the single demonstration generalizes over objects.
Learning-based approaches adapt a learned policy prior to the given demonstration~\cite{finn2017one,  duan2017one, vosylius2024instant}, inheriting both the pretraining requirement and generalize only within its data coverage.
Trajectory-transfer approaches instead directly map the demonstrated motion to a new scene through semantic correspondence~\cite{valassakis2022demonstrate, di2024dinobot, tang2025mimicfunc, liu2025one}, and can transfer a trajectory to novel geometry, appearance, or even a novel object category by identifying corresponding functional parts.
For the first two axes, we find that trajectory transfer is a natural fit.

Successful manipulation, however, requires more than transferring the demonstrated trajectory.
Most prior works on trajectory transfer focus on establishing static correspondences between the demonstration and the target scene, and execute the transferred sequence of actions as an open-loop replay~\cite{valassakis2022demonstrate, di2024dinobot, tang2025mimicfunc, liu2025one}.
The transferred actions must still adapt to the current scene context and execution conditions, where feasible grasp configurations, object placements, and external disturbances may cause the demonstrated execution to become infeasible or deviate from the desired motion.
Closed-loop control observes the object online and corrects such deviations from feedback.
Feedback alone, however, offers no assurance that the corrected motion actually reaches the goal, especially from states far outside the demonstration.
Stable dynamical systems provide exactly this convergence guarantee, continuously correcting deviations while ensuring convergence~\cite{khansari2011learning,lee2025behavior,li2025elastic}.
However, these methods have primarily focused on robust execution given an appropriate state representation and reference trajectory, whereas one-shot trajectory transfer has focused on constructing such trajectories for novel objects.
The integration of these two complementary capabilities, which would encompass all four axes in Table~\ref{tab:comparison}, is still largely unexplored.

To bridge trajectory transfer and stable execution, we present \method{}, a framework built on a unified representation: a set of 3D points on the object.
By anchoring these points directly to the object's functional part rather than the gripper, the policy deduces the desired movements of the designated part in the object-centric space, and can naturally generalize across different embodiments or object categories. 
After the corresponding points are established, the representation compactly summarizes the spatial configurations for physical execution, detached from the semantic reasoning.
Tracking these points under a stable controller closes the loop and provides a convergence guarantee at the goal without requiring dense 3D object geometry.

Building on the point-based representation, \method{} comprises two stages: transfer and execution.
The \textit{transfer} stage grounds the task-specific functional part with a Multi-modal Large Language Model (MLLM) and assigns corresponding points to it.
For example, we can transfer the skill of hanging a mug to a teapot by prompting the functional part, handles, of them, instead of matching the most similar visual feature of the entire object~\cite{zhang2024telling}.
Grounding by name stays robust to large viewpoint changes and broadens the class of target objects a single demonstration can reach.
The point motion is then extracted from the demonstration with a multi-view point tracker robust to large rotations~\cite{rajivc2025multi} and transferred to the target.
When the target goal geometry differs, we adapt the trajectory to the new fixture with a few user-clicked correspondences.
The \textit{execution} stage tracks the same points online and drives them along the transferred trajectory with \textbf{Point-set BCSDM}, our extension of Behavior Controllable Stable Dynamical Model (BCSDM)~\cite{lee2025behavior} from an SE(3) pose to points.
It applies BCSDM to each point and projects the per-point velocities onto a single rigid-body twist for the gripper, and we show that this projection is a gradient descent on a point-set alignment cost, so the terminal phase inherits the almost-global convergence of the classical SO(3) potential without access to the full 3D model.
Since the controller acts on object points rather than the demonstrated gripper pose, \method{} can start from any feasible grasp; we select one by simulating rollouts in advance and rejecting grasps that violate joint limits or collide with the scene.

Experiments in simulation and on a real robot demonstrate that \method{} transfers a single demonstration across object category, grasp pose, and goal geometry, while recovering from external disturbances.
Our contributions are threefold:
(i) We propose \method{}, which attains both generalization and a convergence guarantee from a single demonstration by sharing the point-set representation between transfer and execution.
(ii) We introduce Point-set BCSDM, extending BCSDM~\cite{lee2025behavior} from SE(3) pose trajectories to points, and show that its terminal phase is almost-globally stable.
(iii) We design a name-based functional-part transfer with MLLMs that establishes correspondence under large pose change, where dense feature matching struggles.

\begin{table}[t]
\centering
\caption{One-shot imitation learning approaches for manipulation, compared along the four axes discussed in Sec.~I.}
\label{tab:comparison}
\footnotesize
\setlength{\tabcolsep}{2.5pt}
\renewcommand{\arraystretch}{1.3}
\begin{tabular}{l c c c c}
\toprule
\textbf{Approach} & \textbf{\makecell{No action\\pretrain.}} & \textbf{\makecell{Object\\gen.}} & \textbf{\makecell{Closed\\loop}} & \textbf{\makecell{Conv.\\guarantee}} \\
\midrule
Meta / in-context~\cite{finn2017one, vosylius2024instant} & \xmark & in-dist. & \cmark & \xmark \\
Correspondence~\cite{tang2025mimicfunc, liu2025one} & \cmark & func.\ part$^\ddagger$ & \xmark & \xmark \\
Alignment~\cite{valassakis2022demonstrate, di2024dinobot} & \cmark & intra-cat. & \cmark$^\dagger$ & \xmark \\
Analytical SDS~\cite{khansari2011learning, lee2025behavior} & \cmark & \xmark & \cmark & \cmark \\
\midrule
\textbf{\method{} (ours)} & \cmark & func.\ part$^\ddagger$ & \cmark & \cmark \\
\bottomrule
\multicolumn{5}{l}{\footnotesize $^\dagger$Closed-loop alignment to an intermediate pose, then open-loop replay.} \\
\multicolumn{5}{l}{\footnotesize $^\ddagger$Requires geometrically similar parts.}
\end{tabular}
\end{table}

\begin{figure*}[!ht]
    \centering
     \includegraphics[width=1.0\linewidth]{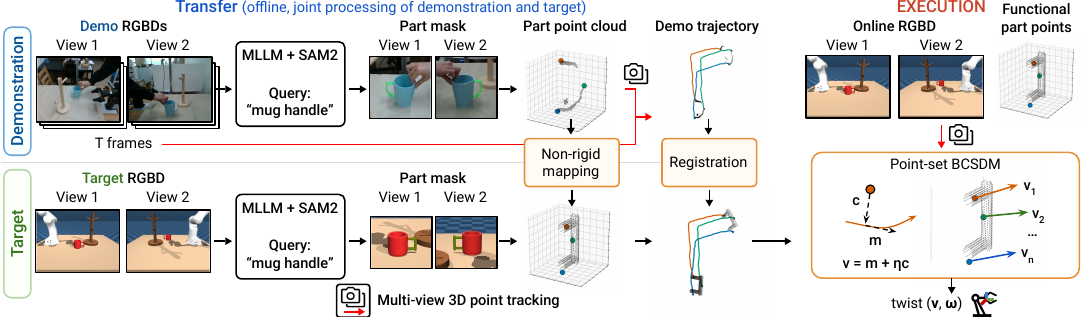}
    \caption{Overview of \method{}. The transfer stage grounds the functional part, establishes point correspondences, and transfers the demonstrated point trajectory to the target scene. The execution stage tracks the same points online and drives them with Point-set BCSDM.}
    \label{fig:pipeline}
\end{figure*}

\section{RELATED WORKS}

\subsection{One-Shot Imitation Learning}
One-shot imitation transfers a skill from a single demonstration to novel scenes.
Existing approaches broadly fall into two directions: learning a generalizable policy from prior data, or directly transferring the demonstrated trajectory to a novel scene.
Learning-based methods adapt from a single demonstration via meta-learning or in-context conditioning~\cite{duan2017one, finn2017one, vosylius2024instant}.
Their generalization comes from pretraining on large-scale action-labeled data, which is far more costly to collect than image or text data.
Thus, the policy's generalizability is still restricted by the coverage of this data, and there is no convergence guarantee.
In contrast, other methods transfer the demonstration without training.
Correspondence-based methods map the demonstration to the target via a shared functional part, generalizing even across categories, but replay the motion open-loop, without correcting deviations during execution~\cite{tang2025mimicfunc, liu2025one}.
Alignment-based methods close the loop only to reach a predefined starting pose, then replay open-loop, and their instance-level alignment confines them to intra-category objects~\cite{valassakis2022demonstrate, di2024dinobot}.

\subsection{Keypoint Representation for Robotic Manipulation}
Keypoints provide a compact and object-centric interface for manipulation~\cite{manuelli2019kpam, gao2021kpam, gao2023k, huang2024rekep, liu2024moka, haldar2025point}.
Obtaining them often relies on supervision specific to a task or object category: kPAM~\cite{manuelli2019kpam} trains a category-specific detector, KETO~\cite{qin2020keto} needs tool-specific demonstrations, and K-VIL~\cite{gao2023k} transfers only within an object category.
More recently, foundation models build keypoint-based task structure with less supervision — spatio-temporal constraints in ReKep~\cite{huang2024rekep}, affordance marks in MOKA~\cite{liu2024moka} — inferring the task without demonstrations, orthogonal to transferring a demonstrated skill.
Closest to ours, RoboTAP~\cite{vecerik2024robotap} tracks points to follow a demonstrated trajectory in closed loop, but on a fixed instance, and Point Policy~\cite{haldar2025point} uses points for both observation and control, but learns a behavior-cloning policy from many demonstrations; neither offers a convergence guarantee.

\subsection{Stable Dynamical Systems}
A stable dynamical system encodes a demonstration as a time-independent velocity field whose solutions converge to the goal from any initial state, correcting deviations along the way.
Built on these formulations, stable dynamical systems have deduced convergence guarantees and been widely adopted in robotic manipulation~\cite{saveriano2023dynamic}.
These methods primarily address the execution problem, assuming that an appropriate state representation and reference trajectory are already available.
Applied to manipulation, the state is a single rigid-body SE(3) pose assigned to one body: either the object, which requires a known 3D model and runtime pose estimation, or the gripper, under the assumption that the object is rigidly held in a fixed grasp~\cite{li2025elastic, lee2025behavior}.
Behavior-Controllable Stable Dynamical Models (BCSDM)~\cite{lee2025behavior} realize such convergence on SE(3) by driving the gripper pose, and interpolates between mimicking and contracting behaviors through a single parameter.
We instead combine stable execution with one-shot trajectory transfer through an object-centric point representation, extending the convergence guarantee to transferred point trajectories without requiring estimated poses or full 3D geometry.
With our semantic transfer stage, the formulation extends the convergence guarantee to novel objects; we also observe empirically that execution tolerates mild non-rigidity (Sec.~\ref{exp:control}, Hanging Bag), though our analysis assumes a rigid object.

\section{METHOD}
\method{} takes a single demonstration and transfers the manipulation skill to a novel target object that shares a functional part, executing the task via a closed-loop controller with convergence guarantees (see Fig.~\ref{fig:pipeline}).
The two stages share one representation -- a set of 3D points on the object's functional part -- so that a single abstraction yields both the object generalization of the transfer stage and the object-centric convergence guarantee of the execution stage.
Formally, the demonstration is a multi-view RGB-D video sequence composed of frames of image and depth observations $\{(I_{n,\tau}^\text{src}, D_{n,\tau}^\text{src})\}$, where $n \in [1, N]$ is the index to a camera and $\tau\in[1, T]$ is the time stamp.
During the demonstration, the robot or a human hand manipulates an object.
At test time, we observe a target scene $(I_{n}^\text{tgt}, D_{n}^\text{tgt})$ with a novel object, where the robot tries to grasp the target object and drives the functional part along the transferred trajectory.

\subsection{Point Trajectory Transfer}
\label{method:point_trajectory_transfer}

\method{} transfers demonstrations through functional-part point trajectories, which capture task-relevant object motion while remaining independent of gripper pose and whole-object appearance.
We obtain this object-centric trajectory by detecting the functional part, aligning the demonstration and target part point clouds, and recovering point motion from the demonstration.
When the target fixture differs from the demonstration, we further adapt the trajectory to the new fixture.

We localize the functional part by its semantic label rather than its appearance, whereas dense image features drift and fail to match across such changes~\cite{zhang2024telling}.
Given demonstration and target object labels and a shared part label, we detect the part with Rex-Omni~\cite{jiang2025detect}, an MLLM-based open-vocabulary detector that returns a bounding box from an image and a text label.
In each scene, we query every view with the object label to crop the object.
Querying each crop separately for the part, however, often returns a spurious box when the part is self-occluded in that view (Fig.~\ref{fig:tvp}).
We therefore introduce \textit{multi-view prompting}: we concatenate all $N$ object crops into a single image grid and issue one part query, which suppresses these errors.
We convert the bounding boxes into part masks using SAM2~\cite{ravi2024sam}, and unproject them with depth into 3D part point clouds $\mathcal{P}_{\text{demo}}$ and $\mathcal{P}_{\text{tgt}}$.
On $\mathcal{P}_{\text{demo}}$ we sample $N_k$ points $\{\mathbf{x}_i^{\text{demo}}\}_{i=1}^{N_k}$ by farthest-point sampling, the object-centric point-set we transfer in the next steps.

\begin{figure}
    \centering
    \includegraphics[width=1.0\linewidth]{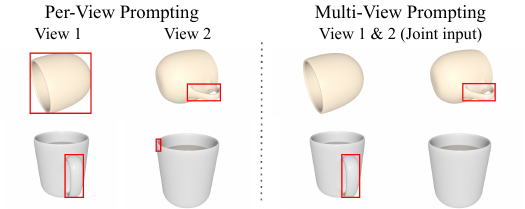}
    \caption{Handle detection from multi-view mug images. Multi-view prompting enables Rex-Omni to localize the handle using cross-view context.}
    \label{fig:tvp}
\end{figure}

With the functional part localized, the remaining problem is geometric: assuming the functional part keeps a similar structure across objects, a dense correspondence between the demonstration and target parts exists to transfer the points.
We recover this correspondence by non-rigid registration. 
Without a shared frame, however, it is ill-posed, collapsing to the least-deformation solution instead of the correct one.
We therefore first fit a coarse transform $T$ -- a rigid transform $(R, \mathbf{t}) \in SE(3)$ with anisotropic scaling along the principal axes of $\mathcal{P}_{\text{demo}}$.
We initialize $T$ by searching rotation and translation candidates, scoring each by Chamfer distance between $\mathcal{P}_{\text{demo}}$ and $\mathcal{P}_{\text{tgt}}$, then refine the top $K$ by gradient descent and select the best as $T$~\cite{kim2025learning}.
We then refine the alignment with occlusion-aware non-rigid registration (OAReg)~\cite{zhao2025occlusion}, which optimizes a deformation field over the coarsely aligned parts.
Passing the sampled demonstration points $\{\mathbf{x}_i^{\text{demo}}\}_{i=1}^{N_k}$ through this field yields the target points $\{\mathbf{x}_i^{\text{tgt}}\}_{i=1}^{N_k}$.

After establishing which functional-part points correspond, we recover how those points move in the demonstration.
This temporal point trajectory is the object-centric motion that will later be followed by the controller.
Because the object can undergo large out-of-plane rotations, single-view point trackers often lose these points~\cite{karaev2024cotracker, karaev2025cotracker3}.
We therefore track them with MVTracker~\cite{rajivc2025multi}, a multi-view RGB-D tracker stable under such motion.
At each timestep, we rigidly align the target points $\{\mathbf{x}_i^{\text{tgt}}\}_{i=1}^{N_k}$ to the tracked demonstration configuration by the Kabsch method~\cite{kabsch1976solution}, yielding transferred point trajectories $\{\mathbf{x}^{\text{tr}}_i(\tau)\}_{i=1}^{N_k}$.
When the target fixture matches the demonstration, these trajectories directly become the controller reference, $\mathbf{x}^{*}_i(\tau) = \mathbf{x}^{\text{tr}}_i(\tau)$.

When the fixture changes shape or pose, directly reusing this trajectory can miss the target geometry, so we non-rigidly adapt it to the new fixture (Fig.~\ref{fig:tps}).
Large fixture deformations make fully automatic registration unreliable~\cite{zhao2025occlusion}, so we ask the user to click a few corresponding point pairs on the two fixtures.
We use the clicked points to segment the fixtures with SAM2~\cite{ravi2024sam} and depth unprojection, then run OAReg~\cite{zhao2025occlusion} on the resulting fixture point clouds.
During OAReg optimization, we add a sparse correspondence loss that aligns the clicked 3D point pairs.
The optimized deformation field yields dense correspondences between the fixtures.
We fit a Thin Plate Spline (TPS)~\cite{bookstein1989principal} $\Phi: \mathbb{R}^3 \to \mathbb{R}^3$ to these correspondences and apply it to every point in the transferred trajectories, yielding the adapted reference $\mathbf{x}^{*}_i(\tau) = \Phi(\mathbf{x}^{\text{tr}}_i(\tau))$.

\begin{figure}[t]
\centering
\includegraphics[width=1.0\linewidth]{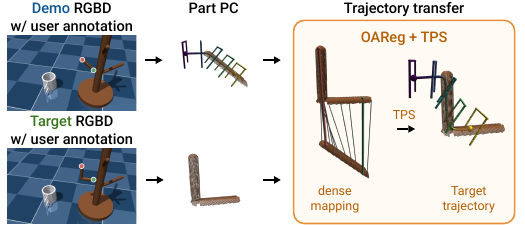}
\caption{Trajectory adaptation to a new fixture. User-clicked correspondences between the demonstration's straight rack (left) and the target's L-shaped rack (right) define a TPS warp $\Phi$. Applying $\Phi$ to every waypoint of the reference trajectory yields the adapted trajectory in the target scene.}
\label{fig:tps}
\end{figure}

\subsection{Closed-Loop Execution via Point-set BCSDM}\label{method:control}
By tracking and controlling functional-part points directly, \method{} executes the transferred trajectory in closed loop with a convergence guarantee, without a known 3D model or run-time pose estimation.
The execution stage proceeds in three steps: (i) offline grasp selection to identify a feasible initial grasp, (ii) online point tracking to update the functional part point states in real time, and (iii) Point-set BCSDM control to drive the object along the transferred trajectory.

\textbf{Grasp selection.}
\method{}'s object-centric point representation provides inherent grasp invariance during execution.
Because the controller acts on object points rather than the demonstrated gripper pose, the robot can select any feasible grasp on the novel object, avoiding the demonstrated grasp when it is in collision or violates joint limits.
When the functional part is symmetric under a 180$^\circ$ rotation (detected by a Chamfer test after flipping about each axis) about one of its principal axes, the point correspondence from Sec.~\ref{method:point_trajectory_transfer} is ambiguous: the mirrored matching fits the part equally well but yields a different reference trajectory.
We therefore roll out each candidate grasp under every symmetric matching as well.
Along each simulated path we check inverse-kinematics feasibility and approximate robot-scene collision between robot link spheres and the scene point cloud.
Each grasp-matching pair is scored by joint-limit and collision violations, and the lowest-cost pair is selected for execution.

\textbf{Online point tracking.}
The point-based representation gives the controller a feedback state on a novel object without estimating its full pose or 3D model.
We track the transferred points online with the same multi-view point tracker~\cite{rajivc2025multi} used in the transfer stage.
To improve robustness, we additionally track points sampled on the whole object as support points~\cite{karaev2024cotracker}, which stabilize the part-point estimates under occlusion and large rotations, while only the part points drive the controller.
Because the tracker runs at a lower frequency than the control loop, we propagate the point positions between tracker updates using the end-effector motion computed from robot forward kinematics (FK), and reset them to the tracker's output whenever a new frame arrives.

\textbf{Point-set BCSDM controller.}
Point-set BCSDM turns the reference point trajectory into closed-loop velocities for the functional-part points.
Given reference trajectories $\{x_i^*(\tau)\}_{i=1}^{N_k}$ and current point positions $\{x_i\}_{i=1}^{N_k}$, we index the targets by $\tau$ and define
\begin{equation}\label{eq:tau}
\tau^* = \argmin_{\tau} \sum_i \bigl\| x_i - x_i^*(\tau) \bigr\|^2
\end{equation}
as the closest target configuration to the current one.
Each point is driven by a velocity field
\begin{equation}\label{eq:vfield}
v_i = \dot{x}_i^*(\tau^*) + \eta\,\bigl( x_i^*(\tau^*) - x_i \bigr),
\end{equation}
which interpolates between mimicking the target velocity ($\eta \to 0$) and contracting toward it (large $\eta$), following BCSDM~\cite{lee2025behavior}.
The robot holds the object with a fixed grasp, so it can only move all points together as one rigid motion.
We therefore find the single twist $(v, \omega)$ that best matches the per-point velocities $\{v_i\}$ in the least-squares sense,
\begin{equation}\label{eq:twist}
(v, \omega) = \argmin_{v, \omega} \sum_i
\bigl\| v + \omega \times (x_i - \bar{x}) - v_i \bigr\|^2,
\end{equation}
about the centroid $\bar{x} = \frac{1}{N_k} \sum_i x_i$. 
In implementation, we robustly integrate this fit with RANSAC~\cite{fischler1981random} to reject outlier point velocities before the final least-squares solve.
We then map this twist to the gripper wrist through the fixed grasp transform and send it to the robot.

\textbf{Stability.}
We assume a rigid object, $x_i = R\,x_i^{\mathrm{b}} + t$ with fixed body-frame points $x_i^{\mathrm{b}}$, and analyze the terminal phase, where $\tau^* = T$, $\dot{x}_i^*(\tau^*) = 0$, and $\eta$ only rescales time, so we set $\eta = 1$.
We analyze the least-squares solve~\eqref{eq:twist} without RANSAC; under the rigid-object assumption all per-point velocities agree with a single twist, so RANSAC rejects nothing.
The twist decreases
\begin{equation}\label{eq:D}
D = \frac{1}{2} \sum_i \bigl\| x_i - x_i^*(\tau^*) \bigr\|^2,
\end{equation}
which we verify below.
Translation is a strictly convex quadratic and trivially stable, so we focus on rotation.
Taking the body frame at the target ($x_i^* = x_i^{\mathrm{b}}$, hence $R = I$ at the target) and centering ($\sum_i x_i^{\mathrm{b}} = 0$),
\begin{equation}\label{eq:DR}
D(R) = \operatorname{tr}(W) - \operatorname{tr}(WR),
\qquad
W = \sum_i x_i^{\mathrm{b}} x_i^{\mathrm{b}\top} \succ 0.
\end{equation}
Let $[v]$ denote the skew-symmetric matrix with $[v]\,u = v \times u$ and $(\cdot)^\vee$ its inverse.
Substituting $v_i = x_i^{\mathrm{b}} - R\,x_i^{\mathrm{b}}$ into~\eqref{eq:twist}, the normal equations in the body frame have Gram matrix $A = \operatorname{tr}(W)\,I - W \succ 0$, the unit-mass inertia of the point set, and right-hand side equal to the skew part of $WR$.
The closed-loop rotation is therefore $\dot{R} = R\,[\omega(R)]$ with
\begin{equation}\label{eq:omega}
\omega(R) = A^{-1}\big(R^\top W - W R\big)^{\vee},
\quad
\dot{D} = -\|\omega(R)\|_A^2 \le 0,
\end{equation}
where $\|u\|_A^2 = u^\top A\, u$.
That is, the least-squares twist is the gradient descent of $D$ in the metric $A$, the kinetic-energy-metric gradient of~\cite{koditschek1989application} with the point-set inertia in place of the body inertia.
$D$ is the potential $\operatorname{tr}(W(I-R))$ proposed in~\cite{koditschek1989application} as a navigation function on $SO(3)$.
Unlike the geodesic cost of BCSDM~\cite{lee2025behavior}, it has three critical points besides $I$; it is known that none is a local minimum~\cite{chillingworth1983symmetry}, so the gradient flow converges almost-globally~\cite{koditschek1989application,chaturvedi2011rigid}.
We restate the argument for our point-set parameterization, where $W$ comes from tracked points only, assuming general position: $W$ has distinct positive eigenvalues $\lambda_1 < \lambda_2 < \lambda_3$.
\begin{theorem}\label{thm:critical_points}
$\omega(R) = 0$ iff $R \in \{I, R^*_1, R^*_2, R^*_3\}$, where $R^*_k$ is the rotation by $\pi$ about the $k$-th principal axis of $W$.
\end{theorem}
This is the $P = W$ case of~\cite{koditschek1989application}, due to~\cite{chillingworth1983symmetry}; we include a short proof.

\begin{proof}
$\omega(R) = 0$ iff $WR = R^\top W = (WR)^\top$, so $WR$ is symmetric and $(WR)^2 = W R R^\top W = W^2$.
In the principal frame $W^2$ has distinct diagonal entries, so $WR$, which commutes with $W^2$, is diagonal with entries $\pm\lambda_k$.
Then $R = W^{-1}(WR) = \operatorname{diag}(\pm1,\pm1,\pm1)$ and $\det R = 1$ leaves the four rotations stated.
\end{proof}
\begin{lemma}\label{lem:escape}
Each $R^*_k$ is a nondegenerate critical point of $D$ and not a local minimum.
\end{lemma}
\begin{proof}
Expanding $D(R^*\exp([\xi]))$ to second order with $[\xi]^2 = \xi\xi^\top - \|\xi\|^2 I$ and $\operatorname{tr}(WR^*[\xi]) = 0$ gives Hessian $H(R^*) = \operatorname{tr}(WR^*)I - WR^*$, the Morse-index matrix of~\cite{chillingworth1983symmetry}.
In the principal frame $WR^*_k$ is diagonal with $\lambda_k$ in position $k$ and $-\lambda_i, -\lambda_j$ elsewhere ($\{i,j\} = \{1,2,3\}\setminus\{k\}$), so $H(R^*_k)$ is diagonal with entries $-\lambda_i - \lambda_j < 0$, $\lambda_k - \lambda_j$, and $\lambda_k - \lambda_i$.
The last two are nonzero by distinctness, so $H(R^*_k)$ is nonsingular, and the first is negative, so $R^*_k$ is not a minimum.
\end{proof}
\begin{theorem}\label{thm:agas}
The rotational dynamics are almost-globally asymptotically stable at $R = I$.
\end{theorem}
\begin{proof}
By Theorem~\ref{thm:critical_points} and Lemma~\ref{lem:escape}, and since $H(I) = A \succ 0$, $D$ is a Morse function on $SO(3)$ with the unique minimum $I$, and~\eqref{eq:omega} is its gradient flow in the metric $A$.
The stable manifolds of the three non-minimum critical points have dimension at most two, so all trajectories outside a measure-zero set converge to $I$~\cite{koditschek1989application}.
\end{proof}

\section{EXPERIMENTS}

We organize the evaluation into four studies.
We isolate the performance of functional-part point transfer (Sec.~\ref{exp:keypoint_transfer}) and control (Sec.~\ref{exp:control}) with ground-truth inputs.
We then evaluate the full pipeline (Sec.~\ref{exp:e2e}) across pose, grasp, object, and goal variations, as well as external disturbances.
We use $N_k=10$ functional-part points, $\eta=3$, and 40 support points for all experiments.

\begin{table}[t]
\centering
\caption{Control robustness under varying initial poses.
Success rate (\%) over 100 trials.
Near denotes the region used to collect demonstrations for PP-50, and Far samples poses outside it.}
\label{tab:control}
\footnotesize
\setlength{\tabcolsep}{3pt}
\renewcommand{\arraystretch}{1.1}
\begin{tabular}{l c cc cc cc cc}
\toprule
\multirow{2}{*}[-0.4ex]{Policy} & \multirow{2}{*}[-0.4ex]{\#Demos}
& \multicolumn{2}{c}{PnP Cube}
& \multicolumn{2}{c}{Mug Insert.}
& \multicolumn{2}{c}{Reshelving}
& \multicolumn{2}{c}{Hanging Bag} \\
\cmidrule(lr){3-4} \cmidrule(lr){5-6} \cmidrule(lr){7-8} \cmidrule(lr){9-10}
& & Near & Far & Near & Far & Near & Far & Near & Far \\
\midrule
PP~\cite{haldar2025point}       & 1  & 100 & 54  & 33  & 25  & 71  & 31  & 51  & 45 \\
PP~\cite{haldar2025point}       & 50 & 100 & 100 & 100 & 84  & 100 & 57  & 100 & 100 \\
IP~\cite{vosylius2024instant}   & 1  & 100 & 100 & 0   & 0   & 4   & 3   & 23  & 19 \\
\midrule
Ours & 1 & 100 & 100 & 100 & 100 & 100 & 100 & 100 & 100 \\
\bottomrule
\multicolumn{10}{l}{\footnotesize PP: Point Policy~\cite{haldar2025point}, IP: Instant Policy~\cite{vosylius2024instant}.}
\end{tabular}
\end{table}

\begin{figure}[t!]
    \centering
    \includegraphics[width=1.0\linewidth]{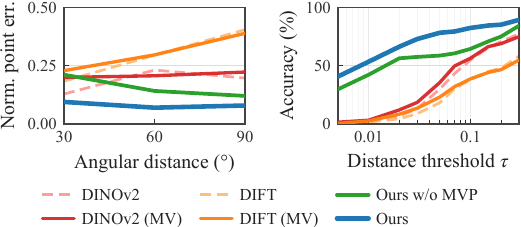}
    \caption{Functional-part point transfer under viewpoint changes.
    We report part point error (left) and the fraction of transferred points within a distance threshold (right).
    The ablation without multi-view prompting (MVP) isolates the effect of cross-view reasoning.}
    \label{fig:semantic_transfer}
\end{figure}

\subsection{Functional-part Point Transfer}
\label{exp:keypoint_transfer}
We evaluate whether our functional-part point transfer remains accurate under large viewpoint and pose changes.

\textbf{Baselines.}
We compare against dense feature matching baselines that rely only on visual similarity.
Single-view baselines lift 2D correspondences to 3D using depth, with features from DINOv2~\cite{oquab2023dinov2} and DIFT~\cite{tang2023emergent}.
Multi-view baselines aggregate 2D features from two views into a 3D feature point cloud following D$^3$Fields~\cite{wang2024d3fields}, then transfer keypoints by nearest neighbor matching.
We pair this aggregation with both DINOv2 and DIFT features.
\method{} is given the object and functional-part names used for prompting; this matches our setting, where the task specifies which part should transfer, and lets us isolate whether name-based part grounding improves correspondence under pose change.
We also ablate multi-view prompting by querying Rex-Omni~\cite{jiang2025detect} independently on each view.

\textbf{Evaluation setup.} 
As no standard evaluation exists for multi-view functional-part point transfer, we construct a controlled setting.
We render objects from three categories (mug, pan, scissors) at 12 viewpoints spaced $30^\circ$ apart on a circle around the object.
For each object, we treat one view as the demonstration and evaluate target views at azimuth offsets from $30^\circ$ to $90^\circ$.

\textbf{Quantitative results.}
Fig.~\ref{fig:semantic_transfer} reports normalized point error in units of the object's longest bounding-box dimension and the fraction of transferred points within a distance threshold.
Ours achieves the lowest point error at every angular offset and the highest correct-point fraction at every threshold.
This indicates that grounding the functional part by name is more robust to pose change than matching dense visual features alone.
The per-view ablation confirms that multi-view prompting suppresses the false detections described in Sec.~\ref{method:point_trajectory_transfer} (Fig.~\ref{fig:tvp}).

\begin{figure}[t!]
    \centering
    \includegraphics[width=1.0\linewidth]{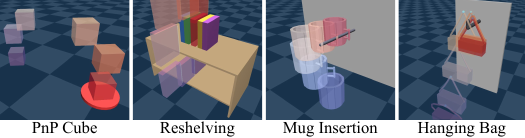}
    \caption{Demonstration sequences for the control tasks. Object motion is rendered from blue to red.}
    \label{fig:control_task}
\end{figure}

\begin{figure}[t!]
    \centering
    \includegraphics[width=1.0\linewidth]{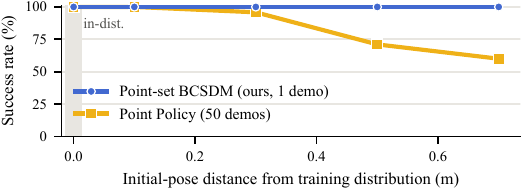}
    \caption{Success rate on Mug Insertion versus the initial pose's xy distance to the training region (20\,cm bins; shaded: in-distribution).
The 50-demo Point Policy degrades with distance, while ours stays at 100\%.}
    \label{fig:control_plot}
\end{figure}

\subsection{Pose-Robust Control with Point-set BCSDM}
\label{exp:control}
We evaluate the core control property of Point-set BCSDM: reaching the demonstrated goal from poses far outside the demonstration distribution.
To isolate control from perception and manipulation constraints, all policies receive ground-truth point trajectories and directly actuate the object.
We compare against two learned baselines.
Point Policy~\cite{haldar2025point} is a behavior cloning method with a point representation; with a single demonstration it memorizes the world-frame trajectory.
Instant Policy~\cite{vosylius2024instant} conditions on the demonstration through a motion prior learned from a large synthetic dataset dominated by linear interpolation between waypoints.

We evaluate on four tasks of increasing difficulty (Fig.~\ref{fig:control_task}).
\textit{Pick-and-Place Cube} (PnP Cube) requires a single rigid translation.
\textit{Reshelving} requires pulling a book out along a specific direction and placing it on an upper shelf.
\textit{Mug Insertion} requires inserting a mug handle onto a rack, which demands following the demonstration trajectory rather than taking a shortcut to the goal pose.
\textit{Hanging Bag} tests extension to a slightly deformable object. 
This is feasible for all three methods because they rely on points rather than object pose.

For each task, we sample initial poses from two regimes with full yaw range $[-180^\circ, 180^\circ]$ and different xy extent.
\textit{Near} draws the xy position from a compact box around the demonstration start.
\textit{Far} extends this box, rejection-sampling every pose to lie outside the \textit{Near} box.
Demonstrations are collected in \textit{Near} for 50-demo Point Policy.
The success criteria depend on task type. 
For \textit{PnP Cube} and \textit{Reshelving}, a trial succeeds if the object center is within 1\,cm of the demonstration goal. 
For \textit{Mug Insertion} and \textit{Hanging Bag}, a trial succeeds if the average distance between current and demonstration goal points is below 1\,cm.

Table~\ref{tab:control} reports success rates.
All three methods solve \textit{PnP Cube}, which reduces to a rigid translation.
Instant Policy fails on \textit{Reshelving} and \textit{Mug Insertion} despite solving \textit{PnP Cube}, so the integration itself is functional.
We attribute this to distribution shift, as its motion prior is trained on RLBench~\cite{james2020rlbench} scenes with near-linear trajectories, though we did not isolate which factor dominates.
Point Policy follows the demonstration when poses stay near it, but with one demonstration it lacks a cue for the current timestep.
Under out-of-distribution poses, it often drives toward the final pose rather than the nearest trajectory configuration, colliding with the wall or rod.
With 50 demonstrations, success remains high near the training distribution but drops with distance (Fig.~\ref{fig:control_plot}), whereas Point-set BCSDM maintains 100\% success across all distances from a single demonstration.

\subsection{End-to-End evaluation}
\label{exp:e2e}
We evaluate this full pipeline end to end under the main variations from the introduction: object displacement, grasp change, goal-geometry change, and cross-category transfer.
In simulation, we use MuJoCo~\cite{todorov2012mujoco} with a 7-DoF Franka Panda arm.
We also evaluate the hanging-rack task on hardware (Fig.~\ref{fig:realworld_exp}).

\begin{figure}[t!]
    \centering
    \includegraphics[width=1.0\linewidth]{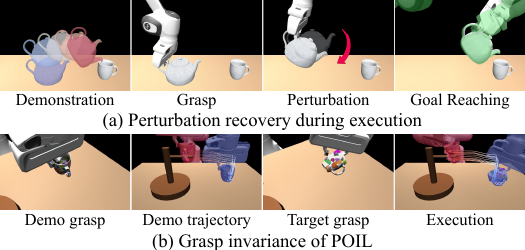}
    \caption{End-to-end execution under object-centric closed-loop control.
    Object motion is rendered from blue to red.
    (a) Teapot pouring with perturbation recovery.
    (b) Mug hanging with a side-lying target mug and a different feasible grasp.}
    \label{fig:end2end}
\end{figure}

\textbf{Object-centric closed-loop control.}
Fig.~\ref{fig:end2end}~(a) shows perturbation recovery in a teapot-pour task; in the final panel, the green overlay marks the state reached without perturbation.
After a displacement, the gripper ends at a different pose while the teapot returns to the same final pour configuration.
In Fig.~\ref{fig:end2end}~(b), the upright mug-hanging demonstration is transferred to a side-lying mug, where the demonstrated rim grasp is blocked by collision and the robot must choose a feasible grasp on the current object pose.
Because control is defined on object points rather than the gripper pose, \method{} still drives the handle to the rack from this different grasp.

\begin{figure}[t!]
    \centering
    \includegraphics[width=1.0\linewidth]{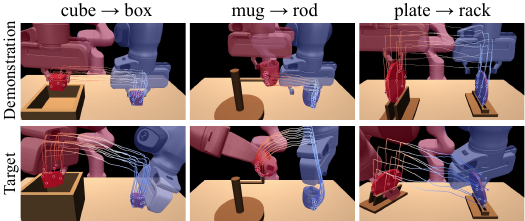}
    \caption{Trajectory adaptation under goal-geometry changes: a cube placed in a taller box, a mug hung on an L-hook rod, and a plate seated in a tilted rack slot.
Top: demonstration. Bottom: execution along the trajectory adapted to the target fixture.}
    \label{fig:goal_adaptation}
\end{figure}

\textbf{Goal transfer.}
When the target fixture differs from the demonstration, OAReg and TPS deform the demonstrated functional-part point trajectory to the target scene (Fig.~\ref{fig:tps}).
\method{} then follows this adapted trajectory in closed loop, producing the executions in Fig.~\ref{fig:goal_adaptation}.
We show three cases: a cube lifted higher than the demonstration to clear a taller box, a plate aligned to a tilted rack, and a mug rotated before insertion to align with an L-shaped rod.
In the plate-rack task, the rack also changes position; because OAReg and TPS align the trajectory to the target fixture, the same mechanism handles this spatial shift.

\textbf{Pipeline failure analysis.}
We evaluate the side-lying mug setting in Fig.~\ref{fig:end2end}(b) over $50$ randomized yaw angles.
POIL completes $39$ trials.
Among the $11$ failures, $7$ result from part-detection failure, and $3$ have no valid grasp after filtering.
The remaining failure is a mug--floor collision, which is not captured by the current gripper-trajectory feasibility check.
These results suggest that part detection under self-occlusion and finding a valid grasp, rather than the downstream point-based controller, are the main bottlenecks in the full pipeline.
Representative executed failures are shown in Fig.~\ref{fig:failure_analysis}.

\begin{figure}[t]
    \centering
    \includegraphics[width=\linewidth]{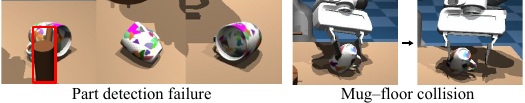}
    \caption{Representative failures in the side-lying mug evaluation. Left: the mug is detected in all views, but the hidden handle causes Rex-Omni~\cite{jiang2025detect} to ground the rack as the handle (red box). Right: a feasible initial grasp leads to mug-floor collision, outside the current feasibility check.}
    \label{fig:failure_analysis}
\end{figure}

\begin{figure}[t]
    \centering
    \includegraphics[width=\linewidth]{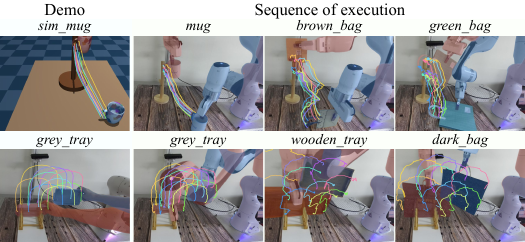}
    \caption{Real-world cross-category transfer from simulation (top) and human hand demonstration (bottom).
    Robot and object motion is rendered from blue (initial) to red (final), with tracked point trajectories overlaid.}
    \label{fig:realworld_exp}
\end{figure}

\begin{figure}[t]
    \centering
    \includegraphics[width=\linewidth]{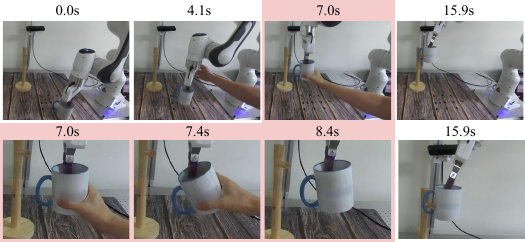}
    \caption{Disturbance recovery during sim-to-real mug hanging. 
    A human pulls the gripper and rotates the mug within the grasp (red; intermediate frames below), yet the handle reaches the demonstrated goal (bottom right).}
    \label{fig:realworld_disturbance}
\end{figure}

\textbf{Real-world cross-category transfer.}
We deploy \method{} on a 7-DoF Franka Emika Panda arm with two ZED 2i depth cameras.
In this setup, the mug is too wide to grasp from the outside, and depth reconstructs its rim poorly.
The flat tray and bags need a human to prop them up before any grasp.
We therefore isolate grasping on hardware: the operator places the object in the gripper at an arbitrary grasp, and we evaluate the rest of the pipeline.
Neither demonstration contains a gripper, as the simulated mug is actuated directly and the tray is moved by a human hand.
Fig.~\ref{fig:realworld_exp} shows the executions.
The top row transfers a simulated mug-hanging demonstration to a real mug, a brown handbag, and a green paper bag lying on the table.
The brown handbag is an interesting case: its elongated handle makes the coarse alignment rotate the transferred points by about 90$^\circ$, yet the execution still hangs the bag, since the task tolerates the handle's in-plane orientation.
The bottom row transfers a human-hand demonstration of placing a grey tray in a dish rack to the same tray, a wooden tray, and a dark bag.
Over 10 sim-to-real mug trials, \method{} succeeds in 6, with 3 tracking failures and 1 infeasible IK solution, and the other transfers are shown qualitatively.
We also add disturbances to this transfer (Fig.~\ref{fig:realworld_disturbance}). 
Even after the mug is rotated within the grasp, the tracked points let the controller drive the handle back to the demonstrated goal.

\section{CONCLUSION AND LIMITATIONS}
We introduced \method{}, which employs functional-part points as a unified representation for one-shot trajectory transfer and stable closed-loop execution.
This closes a gap left by prior trajectory transfer methods that execute open loop and by pose-based dynamical systems that assume a known object-pose or fixed grasp.
In simulation and on a real robot, \method{} transfers a single demonstration across object category, grasp pose, and goal geometry, while recovering from external disturbances.

Several assumptions remain.
First, we assume object and functional-part queries for detection, a few user-clicked correspondences when the fixture changes, and geometrically similar source and target parts for point correspondence.
Second, execution requires the functional-part points to remain trackable; large motions outside the camera views, severe occlusion, or unreliable depth on transparent or specular surfaces can break tracking.
Third, our experiments employ a parallel gripper and a single grasp, which can result in no collision-free, IK-feasible grasp under large pose changes, or no graspable configuration for flat objects.
Future work could divide long transfers into sequential stages using regrasping, or use dexterous hands to reorient the object in hand.

\bibliographystyle{IEEEtran}
\bibliography{references}

\end{document}